\documentclass{article}

  \usepackage[dblblindworkshop, final]{neurips_2026}
\workshoptitle{Geometric Distributional Deep Learning}

\usepackage[utf8]{inputenc} 
\usepackage[T1]{fontenc}    
\usepackage{hyperref}       
\usepackage{url}            
\usepackage{booktabs}       
\usepackage{mathtools, amsmath, amssymb, xcolor, amsthm, enumerate}
\usepackage{nicefrac}
\usepackage{subcaption}
\usepackage{algorithm}
\usepackage{algpseudocode}
\usepackage{tabularx}
\usepackage{makecell}
\usepackage{enumitem}
\usepackage[table]{xcolor}
\usepackage{bbm}
\usepackage{microtype}      
\usepackage{xcolor}         

\newtheorem{theorem}{Theorem}

\newtheorem{proposition}[theorem]{Proposition}
\newtheorem{remark}[theorem]{Remark}

\newcommand{\R}{{\mathbb R}}

\newcommand{\e}{\mathrm{e}}

\DeclareMathOperator{\att}{Att}
\DeclareMathOperator{\lse}{lse}

\newcounter{algline}
\newcommand{\algrow}{\stepcounter{algline}\arabic{algline} & }

\title{Fast Gauss Sums via Flash Attention}

\author{%
  \textbf{Nicolaj Rux} \qquad \textbf{Sebastian Neumayer} \\
  Faculty of Mathematics\\
  Chemnitz University of Technology\\
  Reichenhainer Str.\ 39, 09126 Chemnitz, Germany \\
  \texttt{\{nicolaj.rux, sebastian.neumayer\}@math.tu-chemnitz.de} \\
}

\begin{document}

\maketitle

\begin{abstract}
Gauss kernel sums are the computational core of maximum mean discrepancies (MMDs), kernel gradient flows, Stein variational gradient descent (SVGD), and many other kernel methods.
At the same time, softmax attention has received an extraordinary amount of hardware-aware code engineering, culminating in \texttt{flash} attention.
We show that Gauss kernel sums with arbitrary, signed weights can be evaluated via \texttt{flash} attention:
two small input augmentations turn the normalized softmax reduction into the unnormalized Gauss sum, without writing a single line of custom GPU code.
For feature dimension $D>8$ in \texttt{fp16}, this approach beats compiled PyTorch code in speed and memory, and PyKeOps kernels in speed and accuracy at comparable or lower memory. Its memory scaling remains linear.
\end{abstract}

\section{Introduction}
Let $\Phi_\tau(q,k)=\e^{-\frac\tau2\|q-k\|_2^2}$ denote the rescaled Gauss kernel.
Sums of the form
\begin{equation}\label{eq:ksum}
 \textstyle  s_m \coloneqq \sum_{n=1}^N \Phi_\tau(q_m,k_n) \, v_n \in \R^C,
  \qquad m=1,\dots,M,
\end{equation}
with points $k_n,\,q_m\in\R^D$, values $v_n\in\R^C$ and bandwidth $\tau>0$, appear in essentially every Gauss kernel method.
They are the bottleneck of MMD gradient flows \citep{arbel2019mmd,hertrich2024generative} and of SVGD \citep{liu2016stein}.
Indeed, evaluating \eqref{eq:ksum} naively requires $\mathcal O(MN(C{+}D))$ operations with an $\mathcal O(MN)$ memory footprint.
Classical fast summation methods trade exactness for better asymptotics in restricted regimes \citep{greengard1991fast, beatson1992fast,  yang2004efficient, rahimi2007random, hertrich2024fast}, while PyKeOps \citep{charlier2021kernel} performs the exact reduction based on a fused, memory-efficient GPU kernel.

Meanwhile, a different community has devoted substantial code engineering effort to a single, specific reduction:
softmax attention \citep{vaswani2017attention} was implemented as \texttt{flash} attention \citep{dao2022flashattention,dao2023flashattention2}.
In this note, we show that the Gauss kernel sum \eqref{eq:ksum} reduces to a \texttt{flash} evaluation after two small augmentations, without any custom CUDA or Triton code.

Our contributions are:
(i) two reductions of \eqref{eq:ksum} to attention calls, one fully differentiable through the public PyTorch API and one that additionally reads the logits returned by the attention backends;
(ii) a discussion of the $\mathtt{fp16}$ pitfalls together with simple safeguards;
and (iii) benchmarks against compiled PyTorch and PyKeOps implementations across different input shapes, including gradients.

\section{Gauss kernel sums from flash attention}\label{sec:method}

Softmax attention is denoted by $\att_\tau(q,k,v)$ and given for $m=1,\ldots, M$ via
\begin{equation}\label{eq:att}
\att_\tau(q,k,v)_m\coloneqq\frac{\sum_{n=1}^N \e^{\tau q_m^\top k_n}v_n}{\sum_{n=1}^N \e^{\tau q_m^\top k_n}} \in\R^C ,
  \qquad
  \lse_\tau(q,k)_m \coloneqq \log \sum_{n=1}^N \e^{\tau q_m^\top k_n} \in \R.
\end{equation}
\texttt{Flash} attention evaluates \eqref{eq:att} with $\mathcal O(MN(D{+}C))$ operations and $\mathcal O((M{+}N)(D{+}C))$ memory.
The logits $\lse_\tau$ are computed as a by-product in $\mathtt{fp32}$, but no gradient is implemented.
To connect \eqref{eq:att} with Gauss kernel sums, we remove the normalization by multiplying with $\exp(\lse_\tau(q,k)_m)$ to get
\begin{equation}\label{eq:InterEq}
   \textstyle  \sum_{n=1}^N \e^{\tau q_m^\top k_n} v_n = \att_\tau(q,k,v)_m \cdot \e^{\lse_\tau(q,k)_m}.
\end{equation}
Then, we introduce $\tilde v=(v_n \e^{-\frac{\tau}{2}\|k_n\|^2})_{n=1}^N$ to obtain
\begin{equation}
\frac{\att_\tau(q,k,\tilde v)_m \e^{\lse_\tau(q,k)_m}}{ \e^{\frac{\tau}{2}\|q_m\|^2}}
= \sum_{n=1}^N \e^{\tau q_m^\top k_n} v_n \e^{-\frac{\tau}{2}\|k_n\|^2} \e^{-\frac{\tau}{2}\|q_m\|^2} 
= \sum_{n=1}^N \Phi_\tau(q_m, k_n) v_n.
\end{equation}
Algorithm \ref{alg:lse_prescale} summarizes this method.
Unfortunately, this implementation cannot handle gradients.
\begin{algorithm}[t]
\caption{(\texttt{prescale}) Gauss kernel summation via flash attention }
\label{alg:lse_prescale}
\setcounter{algline}{0}
\renewcommand{\arraystretch}{1.15}
\begin{tabularx}{\linewidth}{@{}r X l l@{}}
\textbf{\#} & \textbf{Step} & \textbf{Work} & \textbf{Call} \\
\hline
\algrow \textbf{Input}
$k\in\R^{N\times D}$, $q\in\R^{M\times D}$, $v\in\R^{N\times C}$, $\tau>0$
& --
&
\\
\algrow \textbf{Output}
$s_m=\sum_{n=1}^N \Phi_\tau(q_m,k_n)v_n$, $m=1,\dots,M$
& --
&
\\
\algrow
$\tilde v_n \coloneqq v_n\exp(-\frac{\tau}{2}\|k_n\|^2)$, $n=1,\ldots, N$
& $N(C+D)$
& mult
\\
\algrow
Compute $\att_\tau(q, k,\tilde v)\in\R^{M\times C}$, $\lse_\tau(q,k)\in \R^M$
& $MN(D{+}C)$
& \bf \texttt{Flash}
\\
\algrow
$s_m\coloneqq \att_\tau(q,k,\tilde v)_m \exp(\lse_\tau(q,k)_m) \exp(-\frac{\tau}{2}\|q_m\|^2)$
& $MC$
& mult
\\
\algrow \textbf{return}
$s$
&
&
\\
\end{tabularx}
\end{algorithm}
Interestingly, we can proceed without using the (non-differentiable) logits by
adding two additional entries to the queries and keys, see Algorithm~\ref{alg:reweight}.
Then, it holds indeed for \smash{$[\alpha, \beta]\coloneqq\att_\tau(\tilde q,\tilde k,\tilde v)$} that

\begin{equation}
\label{eq:reweight}
\frac{\kappa\alpha_m}{\beta_m}
= 
\frac{
\kappa \sum_{n=1}^N \e^{\tau q_m^\top k_n -\frac{\tau}{2} \|k_n\|^2}v_n
}{
\e^{\frac{\tau}{2} \|q_m\|^2}{+}\sum_{n=1}^N \e^{\tau q_m^\top k_n {-}\frac{\tau}{2} \|k_n\|^2}
}
\frac{
\e^{\frac{\tau}{2} \|q_m\|^2}{+}\sum_{n=1}^N \e^{\tau q_m^\top k_n {-}\frac{\tau}{2} \|k_n\|^2}
}{
\e^{\frac{\tau}{2} \|q_m\|^2}\kappa
}
=s_m.
\end{equation}
\begin{proposition}[Stability]
\label{prop:kappa}
In Algorithm~\ref{alg:reweight} (\texttt{reweight}) every $m=1,\ldots, M$ satisfies
\begin{equation}
\kappa(N+1)^{-1} \leq \beta_m \le \kappa,
\qquad
\|\alpha_m\|_\infty \le \max_{n=1,\ldots, N}\|v_n\|_\infty .
\end{equation}
In exact arithmetic this ensures $\beta_m\in[\kappa^{-1},\kappa]$, where  $\kappa=\sqrt{N+1}$.
\end{proposition}

Usually the weights are bounded.
Thus, the critical condition is to ensure that $\beta$ does not underflow. The dynamic range of \texttt{fp16} is $[2^{-14}, 2^{16}]$, which restricts us to $\kappa \le 2^{14}$ or  $N\le  268\,435\,455$.
\begin{algorithm}[H]
\caption{(\texttt{reweight}) Gauss kernel summation via flash attention}
\label{alg:reweight}
\setcounter{algline}{0}
\renewcommand{\arraystretch}{1.15}
\begin{tabularx}{\linewidth}{@{}r X l l@{}}
\textbf{\#} & \textbf{Step} & \textbf{Work} & \textbf{Call} \\
\hline
\algrow \textbf{Input}
$k\in\R^{N\times D}$, $q\in\R^{M\times D}$, $v\in\R^{N\times C}$, $\tau>0$, $\kappa=\sqrt{N{+}1}$
& --
&
\\
\algrow \textbf{Output}
$s_m=\sum_{n=1}^N \Phi_\tau(q_m, k_n)v_n$, $m=1,\dots,M$
& --
&
\\
\algrow
$\tilde q_m \coloneqq [q_m, 1, \frac12\|q_m\|_2^2]\in \R^{D{+}2}$, $m=1,\ldots, M$
& $MD$
& pad
\\
\algrow
$\tilde k_n \coloneqq [k_n,-\frac12\|k_n\|_2^2,0], n=1,\ldots, N$,\quad $\tilde k_0\coloneqq e_{D+2}$
& $ND$
& pad
\\
\algrow
$\tilde v_n \coloneqq [v_n,0], n=1, \ldots, N$,\quad $\tilde v_0\coloneqq \kappa\, e_{C+1}$
& $NC$
& pad
\\
\algrow
$[\alpha, \beta]\coloneqq\att_\tau(\tilde q,\tilde k,\tilde v)$ with $\alpha\in \R^{M\times C}$ and $\beta\in \R^M$
& $MN(D{+}C)$
& \bf \texttt{Flash}
\\
\algrow
$s_m\coloneqq \kappa \alpha _m/\beta_m$, $m=1,\ldots, M$
& $MC$
& mult
\\
\algrow \textbf{return}
$s$
&
&
\\
\end{tabularx}
\end{algorithm}

\begin{remark}[Practical implementation]\label{rem:head}
Algorithm~\ref{alg:reweight} is implemented entirely through the public \texttt{scaled\_dot\_product\_attention} (SDPA) interface: automatic differentiation applies, and the backward pass again runs \texttt{flash} kernels.
Algorithm~\ref{alg:lse_prescale} is algebraically leaner (no auxiliary key, channel or division) but requires the logits, which the PyTorch backends \texttt{flash}, \texttt{cudnn} and \texttt{memory\_efficient} expose.
In practice, we shift $q$ and $k$ by their common mean to avoid numerical overflow.
Moreover, \texttt{flash} requires half precision (\texttt{fp16} or \texttt{bf16}) and a multiple of $8$ for the dimension of queries, keys and values.
\texttt{Flash} is optimized for $D=C\in \{32, 64, 128\}$ and only works up to $D,C\le 256$.
Both implementations therefore run at dimension $D_{\texttt{reweight}} = \lceil\max(D{+}2, C{+}1)\rceil_8$ and $D_{\texttt{prescale}} = \lceil\max(D, C)\rceil_8$, respectively, where $\lceil\cdot\rceil_8$ rounds up to a multiple of $8$.
\end{remark}

\section{Applications}\label{sec:applications}

\paragraph{Maximum mean discrepancy.}
For discrete measures $\mu$ and $\nu$, write their signed difference as
$\sigma \coloneqq \mu-\nu=\sum_{n=1}^N v_n\delta_{k_n}$, with $v\in\R^N$.
Then, their squared maximum mean discrepancy is
\begin{equation}
    \mathrm{MMD}_\tau(\mu,\nu)^2=\iint\Phi_\tau(q,k)\mathrm d\sigma(q)\mathrm d\sigma(k)
  = \sum_{m=1}^N v_m s_m,
  \qquad s_m = \sum_{n=1}^N \Phi_\tau(q_m,k_n)\, v_n.
\end{equation}
Thus, MMD amounts to a kernel sum with $q=k$ and $C=1$, followed by the inner product $v^T s$.

\paragraph{MMD flows and Stein variational gradient descent.}
Particle flows require the gradient of \eqref{eq:ksum} with respect to the query points, namely 
\begin{equation}
    \nabla_{q_m} \sum_{n=1}^N \Phi_\tau(q_m,k_n) v_n
  = \tau\Big(\textstyle\sum_{n=1}^N \Phi_\tau(q_m,k_n) v_n k_n - q_m \sum_{n=1}^N \Phi_\tau(q_m,k_n) v_n\Big),
\end{equation}
which is again a single kernel sum with values $(v_n k_n, v_n)\in\R^{D+1}$.
This is the regime $C= D+1$ and it covers both terms of the SVGD update \cite[Eq.~(8)]{liu2016stein}, whose kernel part has the same form.
Alternatively, plain autograd through Algorithm \ref{alg:reweight} applies.

\section{Numerical results}\label{sec:experiments}

Here, we benchmark different methods for computing \eqref{eq:ksum}, namely a naive \texttt{torch.compile} \citep{ansel2024PyTorch} version (\texttt{PyTorch}), PyKeOps \citep{charlier2021kernel} (\texttt{PyKeOps}) and the two proposed \texttt{flash}-based variants \texttt{reweight} and \texttt{prescale} from Section~\ref{sec:method}.
The inputs are $B$ independent standard Gauss point clouds scaled by \smash{$D^{-1/2}$} with $M=N$ and $\tau=1$, which are evaluated in parallel.
The weights $v$ are i.i.d. standard Gaussian, gradients are taken of \texttt{sum(s)}, and PyKeOps runs with default settings.
For reproducibility, we fix the random seed and always perform $3$ warm-up runs.
Then, we call each method $8$ times and compute the mean and standard deviation of the elapsed time, the memory footprint (excluding input tensors), and the relative $L_2$ error against a \texttt{fp64} reference summation.
All timings were measured on an NVIDIA GeForce RTX 5090 with 32 GB using PyKeOps 2.3, PyTorch 2.13 with CUDA 13.0 and cuDNN 9.20.0.
Since \texttt{flash} relies on Tensor cores running in half precision, we report results for \texttt{fp16}.
Appendix~\ref{app:single_precision} briefly discusses the \texttt{fp32} regime.
\begin{figure}[H]
  \centering
  \begin{subfigure}{0.33\linewidth}
    \includegraphics[width=\linewidth]{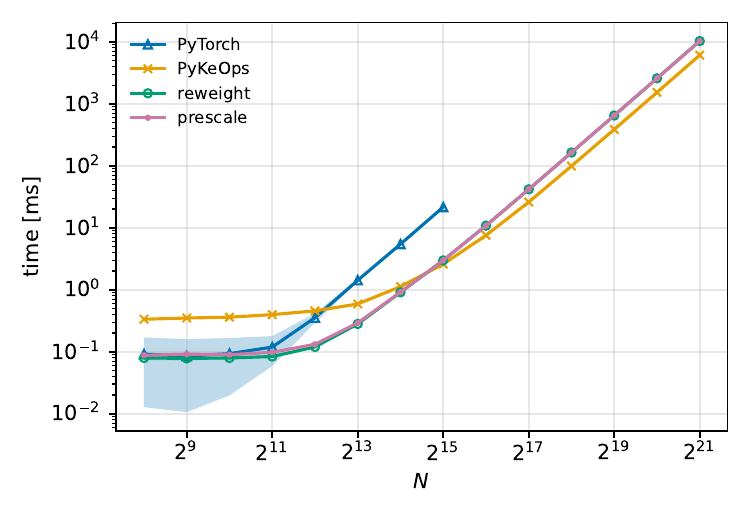}
    \caption{Elapsed time.}\label{fig:sweep_N_a}
  \end{subfigure}\hfill
  \begin{subfigure}{0.33\linewidth}
    \includegraphics[width=\linewidth]{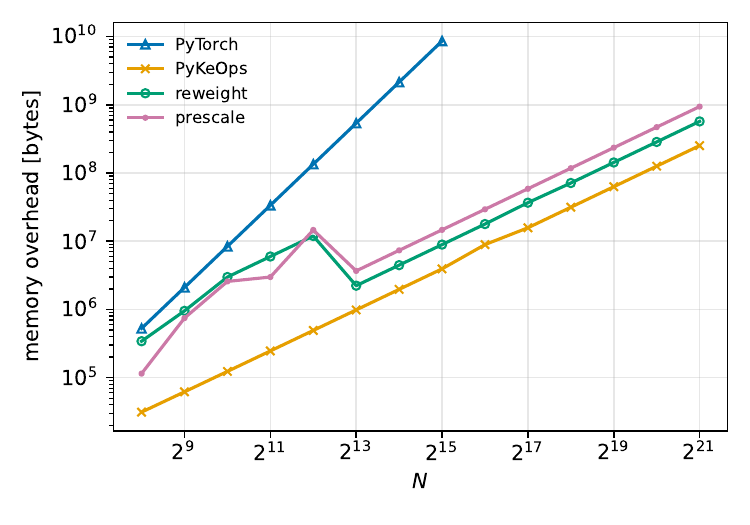}
    \caption{Memory overhead.}\label{fig:sweep_N_b}
  \end{subfigure}\hfill
  \begin{subfigure}{0.33\linewidth}
    \includegraphics[width=\linewidth]{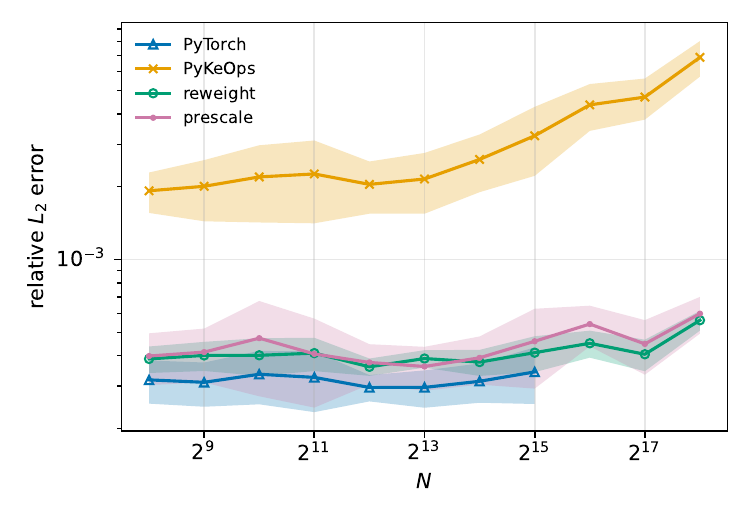}
    \caption{Accuracy. }\label{fig:sweep_N_c}
  \end{subfigure}
  \begin{subfigure}{0.33\linewidth}
    \includegraphics[width=\linewidth]{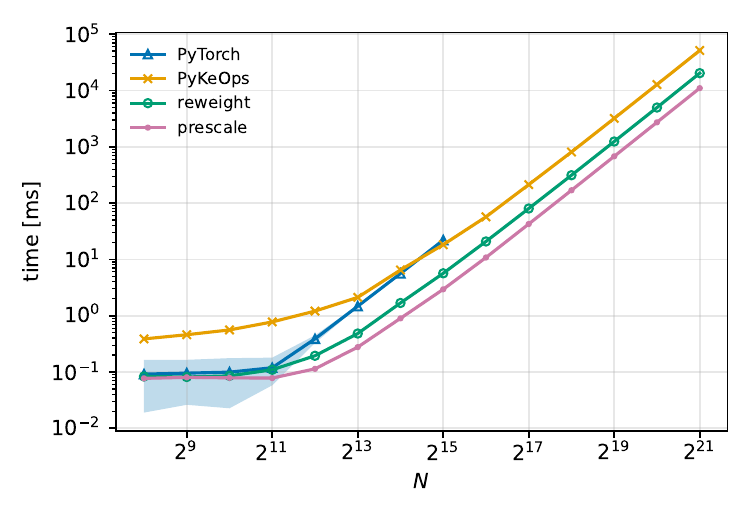}
    \caption{Elapsed time.}\label{fig:sweep_N_d}
  \end{subfigure}\hfill
  \begin{subfigure}{0.33\linewidth}
    \includegraphics[width=\linewidth]{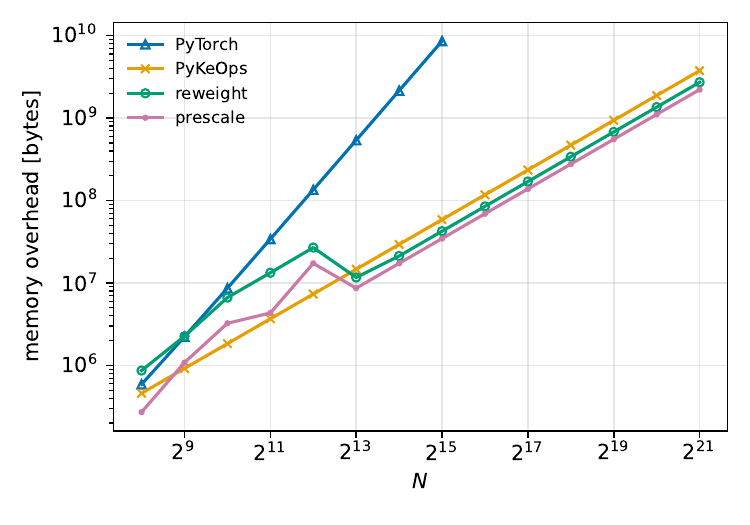}
    \caption{Memory overhead.}\label{fig:sweep_N_e}
  \end{subfigure}\hfill
  \begin{subfigure}{0.33\linewidth}
    \includegraphics[width=\linewidth]{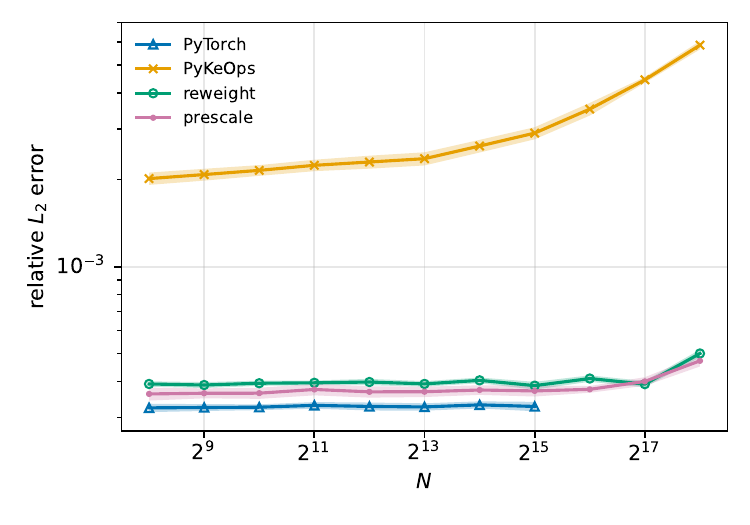}
    \caption{Accuracy.}\label{fig:sweep_N_f}
  \end{subfigure}
  \caption{Sweep over $N$ computing the forward of the Gauss sum \eqref{eq:ksum} for fixed $B=4$  in \texttt{fp16} with backend \texttt{flash}.
  The upper row is $D=3$, $C=1$ and the lower row is $D=32$, $C=32$.}
  \label{fig:sweep_N}
\end{figure}
In Figure~\ref{fig:sweep_N}, we sweep over $N$ with fixed $D$. 
As \texttt{PyTorch} allocates an entire $N\times N$ matrix, this method quickly runs out of memory.
For $D=3$, \texttt{PyKeOps} is marginally faster than \texttt{reweight} and \texttt{prescale} for large $N$.
Further, it uses consistently $3$ times less memory than the \texttt{flash} variants.
However, it has a worse error, especially for large $N$.
For $D=32$, the \texttt{flash} variants outperform \texttt{PyTorch} and \texttt{PyKeOps} in both speed and memory overhead, while maintaining a healthy relative accuracy.
In Figure \ref{fig:sweep_N_backward}, we additionally include the gradient computations for the Gauss sum \eqref{eq:ksum}.
Recall that \texttt{prescale} is not applicable here.
The performance comparison turns out similar to the forward call (see also Figure \ref{fig:sweep_N} bottom row).
\begin{figure}[H]
  \centering
  \begin{subfigure}{0.33\linewidth}
    \includegraphics[width=\linewidth]{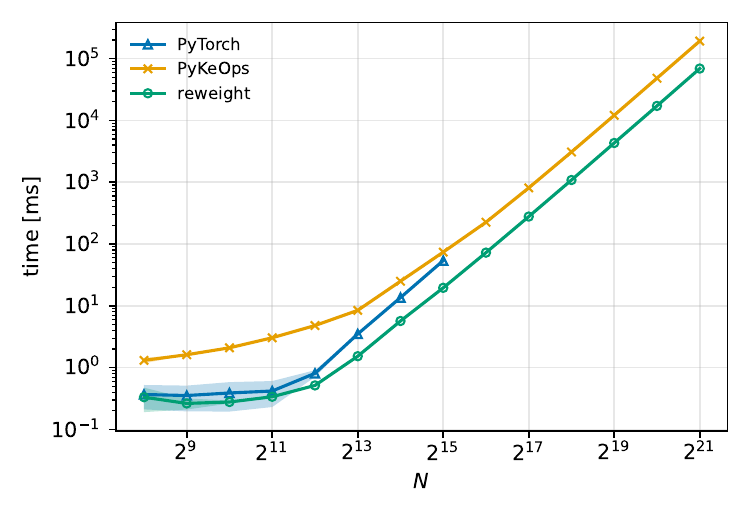}
    \caption{Elapsed time.}\label{fig:sweep_N_backward_a}
  \end{subfigure}\hfill
  \begin{subfigure}{0.33\linewidth}
    \includegraphics[width=\linewidth]{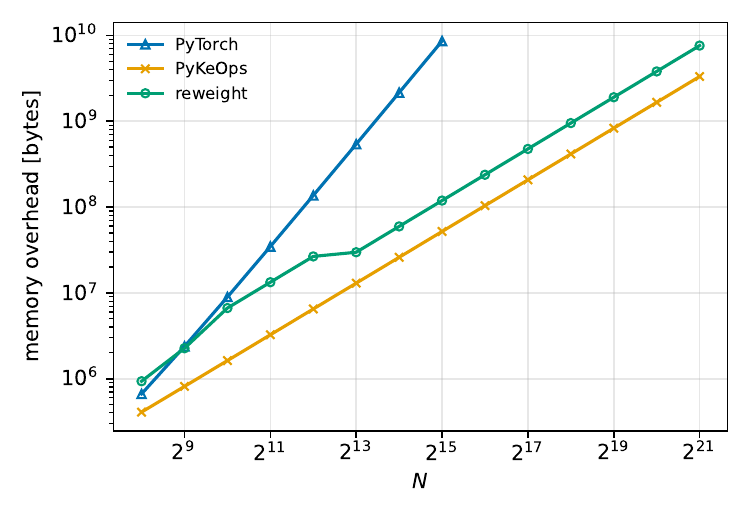}
    \caption{Memory overhead.}\label{fig:sweep_N_backward_b}
  \end{subfigure}\hfill
  \begin{subfigure}{0.33\linewidth}
    \includegraphics[width=\linewidth]{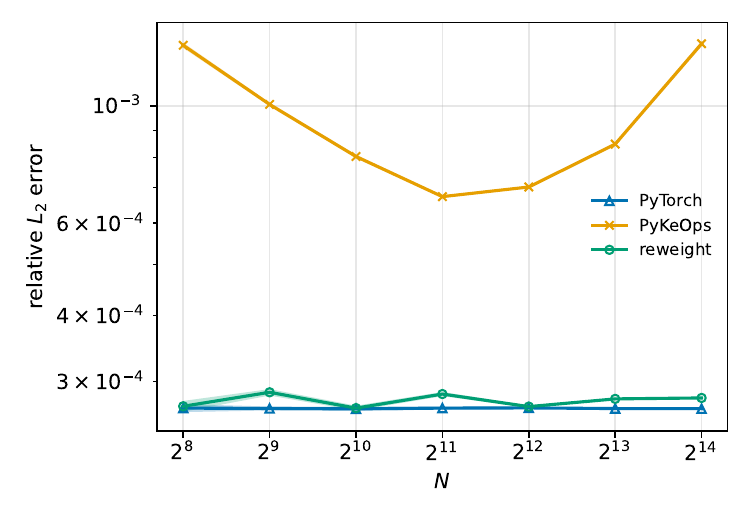}
\caption{Accuracy.}\label{fig:sweep_N_backward_c}
  \end{subfigure}
  \caption{Sweep over $N$ computing the forward and backward of the Gauss sum \eqref{eq:ksum} for fixed $B=4$, $D=32$ and $C=1$ in \texttt{fp16} with backend \texttt{flash}.}
  \label{fig:sweep_N_backward}
\end{figure}
\begin{figure}[H]
  \centering
  \begin{subfigure}{0.33\linewidth}
    \includegraphics[width=\linewidth]{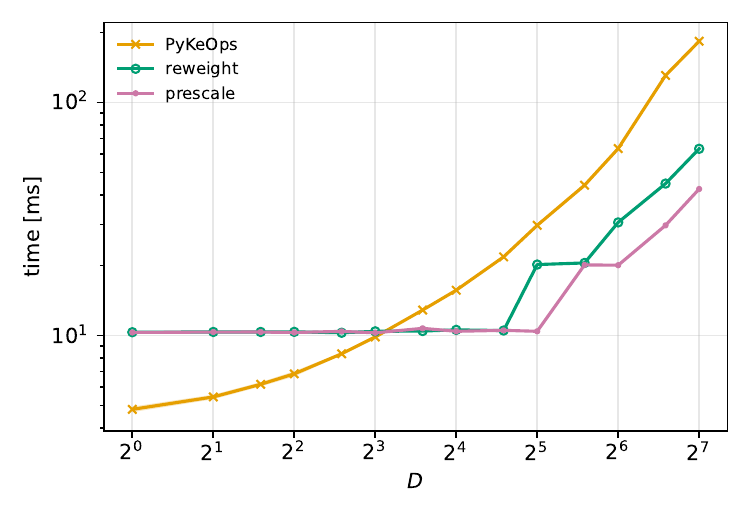}
    \caption{Elapsed time.}\label{fig:sweep_D_a}
  \end{subfigure}\hfill
  \begin{subfigure}{0.33\linewidth}
    \includegraphics[width=\linewidth]{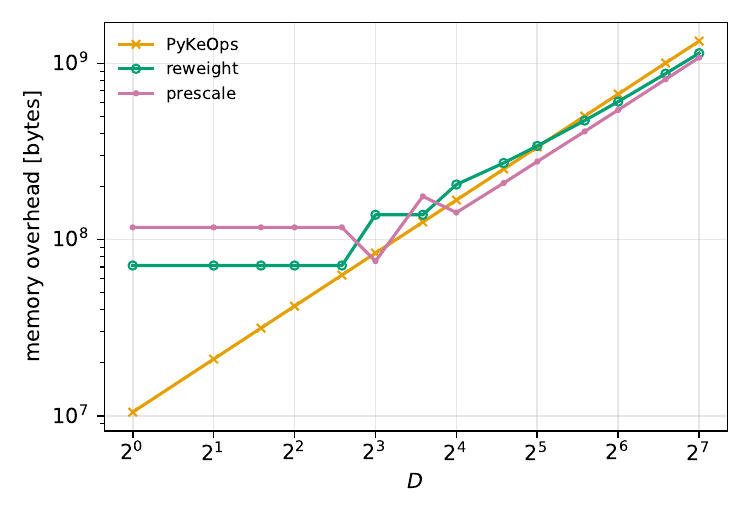}
    \caption{Memory overhead.}\label{fig:sweep_D_b}
  \end{subfigure}\hfill
  \begin{subfigure}{0.33\linewidth}
    \includegraphics[width=\linewidth]{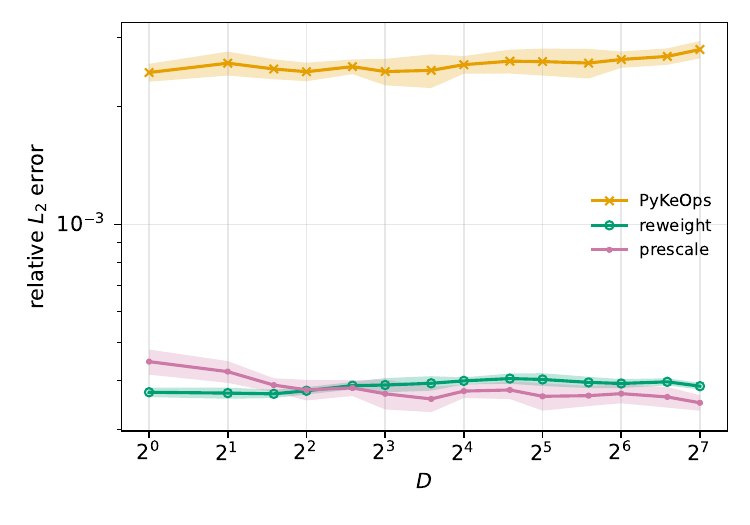}
    \caption{Accuracy.}\label{fig:sweep_D_c}
  \end{subfigure}
  \caption{Sweep over $D$ computing the forward of the Gauss sum \eqref{eq:ksum} for fixed $B=64$, $C=1$ and $N=16384$ in \texttt{fp16} with backend \texttt{flash}.}
  \label{fig:sweep_D}
\end{figure}
In Figure \ref{fig:sweep_D}, the performance across different dimensions $D$ is compared.
\texttt{PyKeOps} consistently has the worst accuracy with an error around $2.7\times 10^{-3}$, while all other methods lie an order of magnitude below around $3.5\times 10^{-4}$.
Its speed degenerates with increasing dimension, while the memory scales linearly and even beats both \texttt{flash} variants for $D\le 8$. Both speed and memory wise \texttt{PyKeOps} is best in the low dimensional regime ($D\le 8$). 
The \texttt{flash} variants are the fastest for $D>8$.
If $D$ is a power of 2 and $C\le D$, \texttt{reweight} is worse compared to \texttt{prescale}, because \texttt{reweight} needs to pad the dimension to $\lceil\max(D{+}2, C{+}1)\rceil_8=D+8$ while \texttt{prescale} only pads to $\lceil\max(D, C)\rceil_8=D$.
This effect becomes visible for $D\ge 16$ in Figure \ref{fig:sweep_D_a}.
For $D\ge 8$, the memory overhead of the \texttt{flash} variants is comparable to the one of \texttt{PyKeOps} with a roughly linear incline.
The accuracy of the \texttt{flash} variants lies around $4.0\times 10^{-4}$ and slightly above \texttt{PyTorch}, which sits at $3.3\times 10^{-4}$.

\section{Conclusion}
Two small input augmentations make \texttt{flash} attention a fast and memory efficient drop-in method for evaluating unnormalized Gauss kernel sums with signed weights. 
It excels at $\mathtt{fp16}$, which often suits sampling, flows and testing, but not ill-conditioned solvers.
For $D\ge 16$, the \texttt{flash}-based Gauss kernel sums run $2{-}21$ times faster than \texttt{PyKeOps} forward and $3{-}10$ times faster with gradients, at up to $7$ times lower error and similar linear memory scaling.
Despite the code optimization, both variants have quadratic complexity.
Thus, combining the constant-factor gains shown here with subquadratic approximations \citep{hertrich2024fast,rux2025numericalmethodskernelslicing} is a promising future direction.
\bibliographystyle{abbrvnat}
\bibliography{bib}

@inproceedings{vaswani2017attention,
       AUTHOR = {Vaswani, Ashish and Shazeer, Noam and Parmar, Niki and
                 Uszkoreit, Jakob and Jones, Llion and Gomez, Aidan N. and
                 Kaiser, {\L}ukasz and Polosukhin, Illia},
    BOOKTITLE = {Advances in Neural Information Processing Systems},
    PUBLISHER = {Curran Associates, Inc.},
        TITLE = {Attention Is All You Need},
          URL = {https://proceedings.neurips.cc/paper_files/paper/2017/file/3f5ee243547dee91fbd053c1c4a845aa-Paper.pdf},
       VOLUME = {30},
         YEAR = {2017}
}

@inproceedings{dao2022flashattention,
       AUTHOR = {Dao, Tri and Fu, Daniel Y. and Ermon, Stefano and
                 Rudra, Atri and R\'{e}, Christopher},
    BOOKTITLE = {Advances in Neural Information Processing Systems},
    PUBLISHER = {Curran Associates, Inc.},
        TITLE = {Flash{A}ttention: Fast and Memory-Efficient Exact Attention
                 with {IO}-Awareness},
          URL = {https://proceedings.neurips.cc/paper_files/paper/2022/file/67d57c32e20fd0a7a302cb81d36e40d5-Paper-Conference.pdf},
       VOLUME = {35},
         YEAR = {2022}
}

@inproceedings{dao2023flashattention2,
       AUTHOR = {Dao, Tri},
    BOOKTITLE = {International Conference on Learning Representations},
        PAGES = {35549--35562},
        TITLE = {Flash{A}ttention-2: Faster Attention with Better Parallelism
                 and Work Partitioning},
          URL = {https://proceedings.iclr.cc/paper_files/paper/2024/file/98ed250b203d1ac6b24bbcf263e3d4a7-Paper-Conference.pdf},
         YEAR = {2024}
}

@article{charlier2021kernel,
   AUTHOR = {Charlier, Benjamin and Feydy, Jean and
             Glaun{\`e}s, Joan Alexis and Collin, Fran{\c{c}}ois-David and
             Durif, Ghislain},
    TITLE = {Kernel Operations on the {GPU}, with Autodiff, without Memory
             Overflows},
  JOURNAL = {J. Mach. Learn. Res.},
 FJOURNAL = {Journal of Machine Learning Research},
   VOLUME = {22},
   NUMBER = {74},
    PAGES = {1--6},
     YEAR = {2021},
      URL = {https://jmlr.org/papers/v22/20-275.html}
}

@inproceedings{arbel2019mmd,
       AUTHOR = {Arbel, Michael and Korba, Anna and Salim, Adil and
                 Gretton, Arthur},
    BOOKTITLE = {Advances in Neural Information Processing Systems},
    PUBLISHER = {Curran Associates, Inc.},
        TITLE = {Maximum Mean Discrepancy Gradient Flow},
          URL = {https://proceedings.neurips.cc/paper_files/paper/2019/file/944a5ae3483ed5c1e10bbccb7942a279-Paper.pdf},
       VOLUME = {32},
         YEAR = {2019}
}

@inproceedings{liu2016stein,
       AUTHOR = {Liu, Qiang and Wang, Dilin},
    BOOKTITLE = {Advances in Neural Information Processing Systems},
    PUBLISHER = {Curran Associates, Inc.},
        TITLE = {Stein Variational Gradient Descent: A General Purpose
                 {B}ayesian Inference Algorithm},
          URL = {https://proceedings.neurips.cc/paper_files/paper/2016/file/b3ba8f1bee1238a2f37603d90b58898d-Paper.pdf},
       VOLUME = {29},
         YEAR = {2016}
}

@article{hertrich2024fast,
   AUTHOR = {Hertrich, Johannes},
    TITLE = {Fast Kernel Summation in High Dimensions via Slicing and
             {F}ourier Transforms},
  JOURNAL = {SIAM J. Math. Data Sci.},
 FJOURNAL = {SIAM Journal on Mathematics of Data Science},
   VOLUME = {6},
   NUMBER = {4},
    PAGES = {1109--1137},
     YEAR = {2024},
      DOI = {10.1137/24M1632085}
}

@inproceedings{hertrich2024generative,
       AUTHOR = {Hertrich, Johannes and Wald, Christian and
                 Altekr{\"u}ger, Fabian and Hagemann, Paul},
    BOOKTITLE = {International Conference on Learning Representations},
        PAGES = {20923--20949},
        TITLE = {Generative Sliced {MMD} Flows with {R}iesz Kernels},
          URL = {https://proceedings.iclr.cc/paper_files/paper/2024/file/5b288823575bb29654b0953a251e933b-Paper-Conference.pdf},
         YEAR = {2024}
}

@inproceedings{ansel2024PyTorch,
       AUTHOR = {Ansel, Jason and Yang, Edward and He, Horace and
                 Gimelshein, Natalia and Jain, Animesh and
                 Voznesensky, Michael and Bao, Bin and Bell, Peter and
                 Berard, David and Burovski, Evgeni and Chauhan, Geeta and
                 Chourdia, Anjali and Constable, Will and Desmaison, Alban and
                 DeVito, Zachary and Ellison, Elias and Feng, Will and
                 Gong, Jiong and Gschwind, Michael and Hirsh, Brian and
                 Huang, Sherlock and Kalambarkar, Kshiteej and
                 Kirsch, Laurent and Lazos, Michael and Lezcano, Mario and
                 Liang, Yanbo and Liang, Jason and Lu, Yinghai and
                 Luk, C. K. and Maher, Bert and Pan, Yunjie and
                 Puhrsch, Christian and Reso, Matthias and Saroufim, Mark and
                 Siraichi, Marcos Yukio and Suk, Helen and Zhang, Shunting and
                 Suo, Michael and Tillet, Phil and Zhao, Xu and Wang, Eikan and
                 Zhou, Keren and Zou, Richard and Wang, Xiaodong and
                 Mathews, Ajit and Wen, William and Chanan, Gregory and
                 Wu, Peng and Chintala, Soumith},
    BOOKTITLE = {Proceedings of the 29th {ACM} International Conference on
                 Architectural Support for Programming Languages and
                 Operating Systems},
          DOI = {10.1145/3620665.3640366},
        PAGES = {929--947},
        TITLE = {Py{T}orch 2: Faster Machine Learning Through Dynamic
                 {P}ython Bytecode Transformation and Graph Compilation},
       VOLUME = {2},
         YEAR = {2024}
}

@article{rux2025numericalmethodskernelslicing,
   AUTHOR = {Rux, Nicolaj and Hertrich, Johannes and Neumayer, Sebastian},
    TITLE = {Numerical Methods for Kernel Slicing},
  JOURNAL = {arXiv preprint},
     YEAR = {2025},
      URL = {https://arxiv.org/abs/2510.11478}
}

@article{greengard1991fast,
   AUTHOR = {Greengard, Leslie and Strain, John},
    TITLE = {The Fast {G}auss Transform},
  JOURNAL = {SIAM J. Sci. Stat. Comput.},
 FJOURNAL = {SIAM Journal on Scientific and Statistical Computing},
   VOLUME = {12},
   NUMBER = {1},
    PAGES = {79--94},
     YEAR = {1991},
      DOI = {10.1137/0912004}
}

@article{beatson1992fast,
   AUTHOR = {Beatson, Richard K. and Newsam, Garry N.},
    TITLE = {Fast Evaluation of Radial Basis Functions: {I}},
  JOURNAL = {Comput. Math. Appl.},
 FJOURNAL = {Computers \& Mathematics with Applications},
   VOLUME = {24},
   NUMBER = {12},
    PAGES = {7--19},
     YEAR = {1992},
      DOI = {10.1016/0898-1221(92)90167-G}
}

@inproceedings{yang2004efficient,
       AUTHOR = {Yang, Changjiang and Duraiswami, Ramani and Davis, Larry S.},
    BOOKTITLE = {Advances in Neural Information Processing Systems},
    PUBLISHER = {MIT Press},
        TITLE = {Efficient Kernel Machines Using the Improved Fast
                 {G}auss Transform},
          URL = {https://proceedings.neurips.cc/paper_files/paper/2004/file/85353d3b2f39b9c9b5ee3576578c04b7-Paper.pdf},
       VOLUME = {17},
         YEAR = {2004}
}

@inproceedings{rahimi2007random,
       AUTHOR = {Rahimi, Ali and Recht, Benjamin},
    BOOKTITLE = {Advances in Neural Information Processing Systems},
    PUBLISHER = {Curran Associates, Inc.},
        TITLE = {Random Features for Large-Scale Kernel Machines},
          URL = {https://proceedings.neurips.cc/paper_files/paper/2007/file/013a006f03dbc5392effeb8f18fda755-Paper.pdf},
       VOLUME = {20},
         YEAR = {2007}
}

\section{Appendix}
\subsection{Proof of Proposition \ref{prop:kappa}}
\label{app:proof:kappa}
\begin{proof}
Since $0< \Phi_\tau(q_m, k_n)\le 1$, the estimate $\kappa(N+1)^{-1} \leq \beta_m \le \kappa$ follows from 
\begin{equation}
\beta_m 
= \frac{\kappa
\e^{ \frac{\tau}{2} \|q_m\|^2}
}{
\e^{\frac{\tau}{2} \|q_m\|^2}{+}\sum_{n=1}^N \e^{\tau q_m^\top k_n -\frac{\tau}{2} \|k_n\|^2}
}
= \frac{\kappa }{1+\sum_{n=1}^N \Phi_\tau(q_m, k_n)}.
\end{equation}
The bound on $\alpha$ follows directly as each $\alpha_m$ is a convex combination of $v$ and $\|\cdot\|_\infty$ is convex.
\end{proof}

\subsection{Single precision}
\label{app:single_precision}
For \texttt{fp32}, \texttt{flash} is no longer available as it relies on highly optimized tensor cores. 
Instead, the default backend for softmax attention is \texttt{memory\_efficient}.
Figure \ref{fig:sweep_D} and Figure \ref{fig:sweep_D_fp32} use the same setup ($B=64$, $C=1$, $N=16384$) but in \texttt{fp16} with \texttt{flash} and \texttt{fp32} with \texttt{memory\_efficient}.
For $D\leq 32$, \texttt{memory\_efficient} in \texttt{fp32} is around $10$ times slower than its \texttt{flash}-\texttt{fp16} counterpart.
In contrast, \texttt{PyKeOps} is optimized for \texttt{fp32}.
Remarkably, its memory overhead remains tiny across dimensions.
While the runtime of \texttt{PyKeOps} does not change much between  \texttt{fp16} and \texttt{fp32}, the softmax based variants become competitive only around $D\ge 32$.
While all methods have a healthy accuracy of around $5\times 10^{-7}$, \texttt{PyKeOps} consistently has the lowest relative $L_2$ error.
Thus, for \texttt{fp32}, \texttt{PyKeOps} remains the state-of-the-art implementation.
\begin{figure}[H]
  \centering
  \begin{subfigure}{0.33\linewidth}
    \includegraphics[width=\linewidth]{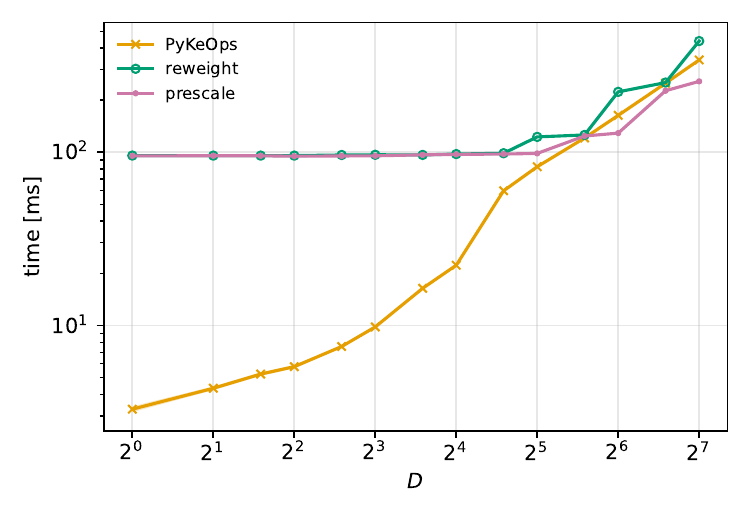}
    \caption{Elapsed time.}\label{fig:sweep_D_a_fp32}
  \end{subfigure}\hfill
  \begin{subfigure}{0.33\linewidth}
    \includegraphics[width=\linewidth]{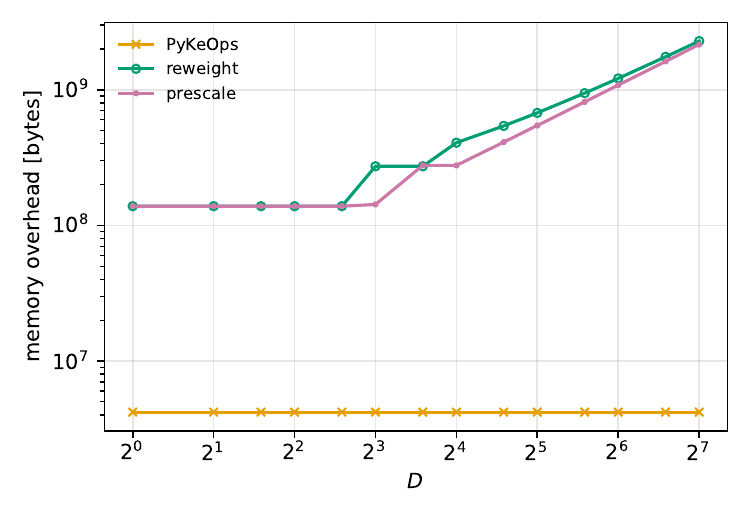}
    \caption{Memory overhead.}\label{fig:sweep_D_b_fp32}
  \end{subfigure}\hfill
  \begin{subfigure}{0.33\linewidth}
    \includegraphics[width=\linewidth]{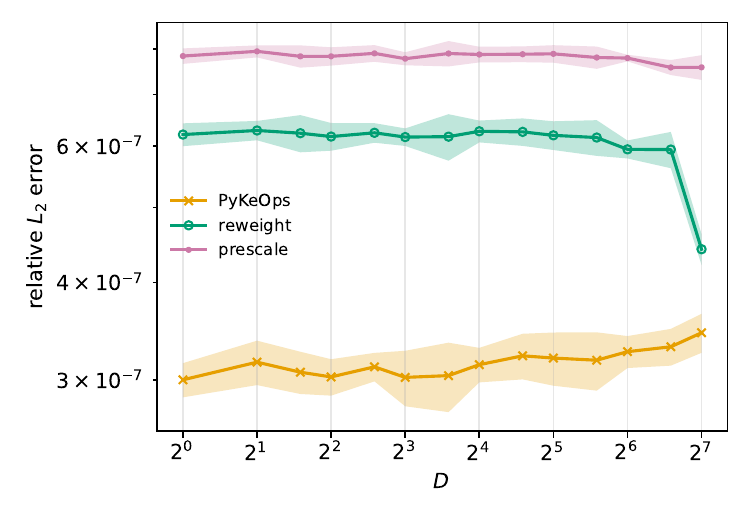}
    \caption{Accuracy.}\label{fig:sweep_D_c_fp32}
  \end{subfigure}
  \caption{Sweep over $D$ computing the forward of the Gauss sum \eqref{eq:ksum} for fixed $B=64$, $C=1$ and $N=16384$ in \texttt{fp32} with backend \texttt{memory\_efficient}.}
  \label{fig:sweep_D_fp32}
\end{figure}

\begin{figure}[H]
  \centering
  \begin{subfigure}{0.33\linewidth}
    \includegraphics[width=\linewidth]{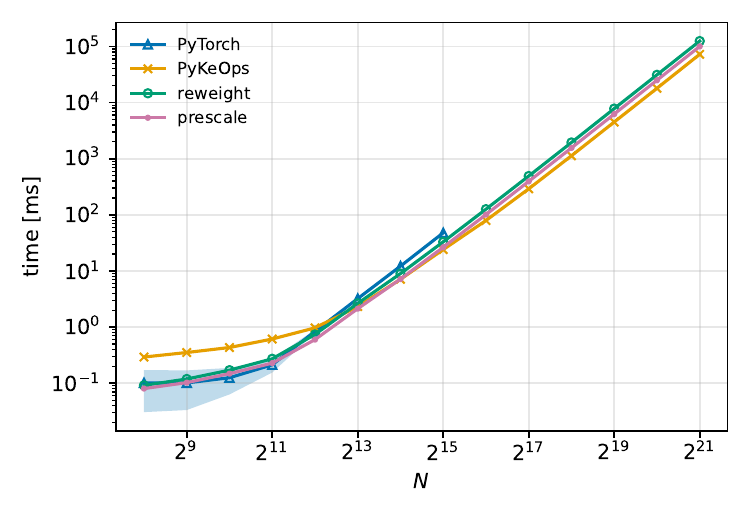}
    \caption{Elapsed time.}\label{fig:sweep_N_a_fp32}
  \end{subfigure}\hfill
  \begin{subfigure}{0.33\linewidth}
    \includegraphics[width=\linewidth]{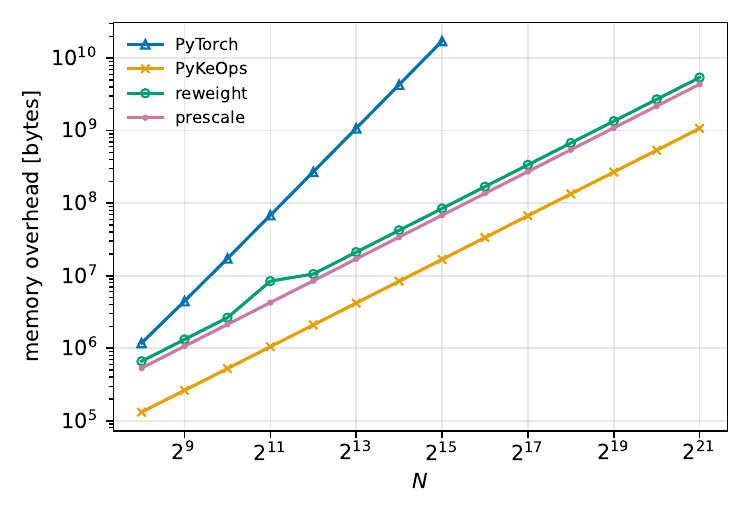}
    \caption{Memory overhead.}\label{fig:sweep_N_b_fp32}
  \end{subfigure}\hfill
  \begin{subfigure}{0.33\linewidth}
    \includegraphics[width=\linewidth]{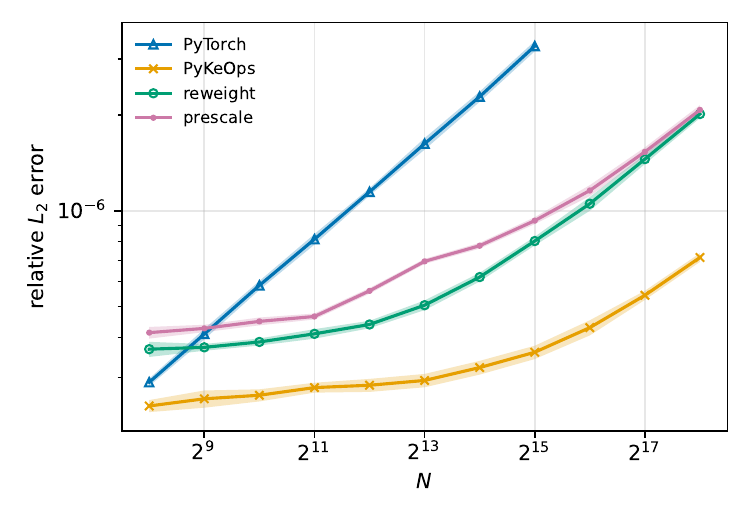}
    \caption{Accuracy.}\label{fig:sweep_N_c_fp32}
  \end{subfigure}
  \caption{Sweep over $N$ computing the forward of the Gauss sum \eqref{eq:ksum} for fixed $B=4$, $D=32$ and $C=32$ in \texttt{fp32} with backend \texttt{memory\_efficient}.}
  \label{fig:sweep_N_fp32}
\end{figure}
\end{document}